\documentclass{article}

\usepackage[final]{neurips_2019}

\usepackage{natbib}
\usepackage{nicefrac}
\usepackage{subfig}
\usepackage{amssymb}
\usepackage{graphics}
\usepackage[pdftex]{graphicx}
\usepackage{amsmath}
\usepackage{amsthm}
\usepackage{sidecap}
\usepackage{url}
\usepackage{booktabs}
\usepackage{microtype}
\usepackage{relsize}
\usepackage{floatrow}
\usepackage{wrapfig}

\usepackage[vlined,ruled,linesnumbered]{algorithm2e}
\newcommand{\alginit}{\SetCommentSty{small}\DontPrintSemicolon}

\newcommand{\note}[1]{}
\renewcommand{\note}[1]{~\\\frame{\begin{minipage}[c]{\columnwidth}\vspace{2pt}\center{#1}\vspace{2pt}\end{minipage}}\vspace{3pt}\\}

\newcommand{\OPT}{\text{OPT}}

\newcommand{\E}{\mathrm{E}}

\newcommand{\eps}{\varepsilon}

\newcommand{\ignore}[1]{}
\newtheorem{theorem}{Theorem}[section]
\newtheorem{lemma}[theorem]{Lemma}

\newtheorem{claim}[theorem]{Claim}
\newtheorem{definition}[theorem]{Definition}
\newtheorem{example}[theorem]{Example}

\newcommand{\run}{{\textsc{Run}}}

\DeclareMathOperator*{\argmin}{arg\,min}
\DeclareMathOperator*{\argmax}{arg\,max}

\begin{document}

\title{Procrastinating with Confidence: Near-Optimal, Anytime, Adaptive Algorithm Configuration}
\author{Robert Kleinberg\\Department of Computer Science\\Cornell University\\\texttt{rdk@cs.cornell.edu}
\And
Kevin Leyton-Brown\\Department of Computer Science\\University of British Columbia\\\texttt{kevinlb@cs.ubc.ca}
\And
Brendan Lucier\\Microsoft Research\\\texttt{brlucier@microsoft.com}
\And
Devon Graham\\Department of Computer Science\\University of British Columbia\\\texttt{drgraham@cs.ubc.ca}}

\maketitle

\begin{abstract}
Algorithm configuration methods optimize the performance of a parameterized heuristic algorithm on a given distribution of problem instances. Recent work introduced an algorithm configuration procedure (``Structured Procrastination'') that provably achieves near optimal performance with high probability and with nearly minimal runtime in the worst case. It also offers an \emph{anytime} property: it keeps tightening its optimality guarantees the longer it is run. Unfortunately, Structured Procrastination is not \emph{adaptive} to characteristics of the parameterized algorithm: it treats every input like the worst case. Follow-up work (``LeapsAndBounds'') achieves adaptivity but trades away the anytime property. This paper introduces a new algorithm, ``Structured Procrastination with Confidence'', that preserves the near-optimality and anytime properties of Structured Procrastination while adding adaptivity. In particular, the new algorithm will perform dramatically faster in settings where many algorithm configurations perform poorly. We show empirically both that such settings arise frequently in practice and that the anytime property is useful for finding good configurations quickly.
\end{abstract}

\section{Introduction}

Algorithm configuration is the task of searching a space of \emph{configurations} of a given algorithm (typically represented as joint assignments to a set of algorithm parameters) in order to find a single configuration that optimizes a performance objective on a given distribution of inputs. In this paper, we focus exclusively on the objective of minimizing average runtime. Considerable progress has recently been made on solving this problem in practice via general-purpose, heuristic techniques such as ParamILS \citep{hutter-aaai07a,hutter-jair09a},
GGA \citep{ansotegui-cp09a,ansotegui-ijcai15a},
irace \citep{birattari-gecco02a,lopez-ibanez-tech11a} and
SMAC \citep{hutter-bayesopt11,hutter-lion11a}. Notably, in the context of this paper, all these methods are \emph{adaptive}: they surpass their worst-case performance when presented with ``easier'' search problems.

Recently, algorithm configuration has also begun to attract theoretical analysis.
While there is a large body of less-closely related work that we survey in Section~\ref{sec:related}, the first nontrivial worst-case performance guarantees for general algorithm configuration with an average runtime minimization objective were achieved by a recently introduced algorithm called \emph{Structured Procrastination (SP)} \citep{ijcai17}.
This work considered a worst-case setting in which an adversary causes every deterministic choice to play out as poorly as possible, but where observations of random variables are unbiased samples. It is straightforward to argue that, in this setting, 
any fixed, deterministic heuristic for searching the space of configurations
can be extremely unhelpful. The work therefore focuses on obtaining candidate configurations via random sampling (rather than, e.g., following gradients or taking the advice of a response surface model).  Besides its use of heuristics, SMAC also devotes half its runtime to random sampling. Any method based on random sampling will eventually encounter the optimal configuration; the crucial question is the amount of time that this will take. The key result of \citet{ijcai17} is that SP is guaranteed to find a near-optimal configuration with high probability, with worst-case running time that nearly matches a lower bound on what is possible and that asymptotically dominates that of existing alternatives such as SMAC.

Unfortunately, there is a fly in the ointment: SP turns out to be impractical in many cases, taking an extremely long time to run even on inputs that existing methods find easy. At the root, the issue is that SP treats every instance like the worst case, in which it is necessary to achieve a fine-grained understanding of every configuration's runtime in order to distinguish between them.  For example, if every configuration is very similar but most are not quite $\eps$-optimal, subtle performance differences must be identified. SP thus runs every configuration enough times that with high probability the configuration's runtime can accurately be estimated to within a $1+\eps$ factor.

\subsection{\textsc{LeapsAndBounds} and \textsc{CapsAndRuns}}
\label{sec:beyond}

\citet{weisz2018leapsandbounds} introduced a new algorithm, \textsc{LeapsAndBounds (LB)}, that improves upon Structured Procrastination in several ways. First, LB improves upon SP's worst-case performance, matching its information-theoretic lower bound on running time by eliminating a log factor. Second, LB does not require the user to specify a runtime cap that they would never be willing to exceed on any run, replacing this term in the analysis with the runtime of the optimal configuration, which is typically much smaller. Third, and most relevant to our work here, LB includes an adaptive mechanism, which takes advantage of the fact that when a configuration exhibits low variance across instances, its performance can be estimated accurately with a smaller number of samples. However, the easiest algorithm configuration problems are probably those in which a few  configurations are much faster on average than all other configurations.  (Empirically, many algorithm configuration instances exhibit just such non-worst-case behaviour; see our empirical investigation in the Supplementary Materials.) In such cases, it is clearly unnecessary to obtain high-precision estimates of each bad configuration's runtime; instead, we only need to separate these configurations' runtimes from that of the best alternative. LB offers no explicit mechanism for doing this. LB also has a key disadvantage when compared to SP: it is not anytime, but instead must be given fixed values of $\eps$ and $\delta$. Because LB is adaptive, there is no way for a user to anticipate the amount of time that will be required to prove $(\eps,\delta)$-optimality, forcing a tradeoff between the risks of wasting available compute resources and of having to terminate LB before it returns an answer.

\textsc{CapsAndRuns (CR)} is a refinement of LB that was developed concurrently with the current paper; it has not been formally published, but was presented at an ICML 2018 workshop \citep{weisz2018capsandruns}.  CR maintains all of the benefits of LB, and furthermore introduces a second adaptive mechanism that does exploit variation in configurations' mean runtimes. Like LB, it is not anytime.

\subsection{Our Contributions}

Our main contribution is a refined version of SP that maintains the anytime property while aiming to observe only as many samples as necessary to separate the runtime of each configuration from that of the best alternative.  We call it ``Structured Procrastination with Confidence'' (SPC). SPC differs from SP in that it maintains a novel form of lower confidence bound as an indicator of the quality of a particular configuration, while SP simply uses that configuration's sample mean. The consequence is that SPC spends much less time running poorly performing configurations, as other configurations quickly appear better and receive more attention. We initialize each lower bound with a trivial value: each configuration's runtime is bounded below by the fastest possible runtime, $\kappa_0$. SPC then repeatedly evaluates the configuration that has the most promising lower bound.\footnote{While both SPC and CR use confidence bounds to guide search, they take different approaches. Rather than rejecting configurations whose lower bounds get too large, SPC focuses on configurations with small lower bounds. By allocating a greater proportion of total runtime to such promising configurations we both improve the bounds for configurations about which we are more uncertain and allot more resources to configurations with relatively low mean runtimes about which we are more confident.} We perform these runs by ``capping'' (censoring) runs at progressively doubling multiples of $\kappa_0$.  If a run does not complete, SPC ``procrastinates'', deferring it until it has exhausted all runs with shorter captimes.  Eventually, SPC observes enough completed runs of some configuration to obtain a nontrivial upper bound on its runtime.  At this point, it is able to start drawing high-probability conclusions that other configurations are worse.


Our paper is focused on a theoretical analysis of SPC. We show that it identifies an approximately optimal configuration using running time that is nearly the best possible in the worst case; however, so does SP. The key difference, and the subject of our main theorem, is that SPC also exhibits near-minimal runtime beyond the worst case, in the following sense. Define an $(\eps,\delta)$-suboptimal configuration to be one whose average runtime exceeds that of the optimal configuration by a factor of  more than $1+\eps$, even when the suboptimal configuration's runs are capped so that a $\delta$ fraction of them fail to finish within the time limit. A straightforward information-theoretic argument shows that in order to verify that a configuration is $(\eps,\delta)$-suboptimal it is sufficient---and may also be necessary, in the worst case---to run it for $O(\eps^{-2} \cdot \delta^{-1} \cdot \OPT)$ time. The running time of SPC matches (up to logarithmic factors) the running time of a hypothetical ``optimality verification procedure'' that knows the identity of the optimal configuration, and for each suboptimal configuration $i$ knows a pair $(\eps_i,\delta_i)$ such that $i$ is $(\eps_i,\delta_i)$-suboptimal and the product $\eps_i^{-2} \cdot \delta_i^{-1}$ is as small as possible. 

SPC is anytime in the sense that it first identifies an  $(\eps,\delta)$-optimal configuration for large values of $\eps$ and $\delta$ and then continues to refine these values as long as it is allowed to run. This is helpful for users who have difficulty setting these parameters up front, as already discussed. SPC's strategy for progressing iteratively through smaller and smaller values of $\eps$ and $\delta$ also has another advantage: it is actually faster than starting with the ``final'' values of $\eps$ and $\delta$ and applying them to each configuration. This is because extremely weak configurations can be dismissed cheaply based on large  $(\eps, \delta)$ values, instead of taking  more samples  to estimate their runtimes more finely.

\subsection{Other Related Work}
\label{sec:related}

There is a large body of related work in the multi-armed bandits literature, which does not attack quite the same problem but does similarly leverage the ``optimism in the face of uncertainty'' paradigm and many tools of analysis \citep{lai1985asymptotically,auer2002finite,bubeck2012regret}. We do not survey this work in detail as we have little to add to the extensive discussion by \citet{ijcai17}, but we briefly identify some dominant threads in that work. Perhaps the greatest contact between the communities has occurred in the sphere of hyperparameter optimization \citep{bergstra2011algorithms,thornton2013auto,li2016hyperband} and in the literature on bandits with correlated arms that scale to large experimental design settings \citep{kleinberg2006anytime,kleinberg2008multi,chaudhuri2009parameter,bubeck2011x,srinivas2012information,cesa2012combinatorial,munos2014bandits,shahriari2016taking}. In most of this literature, all arms have the same, fixed cost; others \citep{guha2007approximation,tran2012knapsack,badanidiyuru2013bandits} consider a model where costs are variable but always paid in full. (Conversely, in algorithm configuration we can stop runs that exceed a captime, yielding a potentially censored sample at bounded cost.) Some influential departures from this paradigm include \citet{kandasamy2016multi}, \citet{ganchev2010censored}, and most notably \citet{li2016hyperband}; reasons why these methods are nevertheless inappropriate for use in the algorithm configuration setting are discussed at length by \citet{ijcai17}.

Recent work has examined the learning-theoretic foundations of algorithm configuration, inspired in part by an influential paper of \citet{gupta2017pac} that framed algorithm configuration and algorithm selection in terms of learning theory. This vein of work has not aimed at a general-purpose algorithm configuration procedure, as we do here, but has rather sought sample-efficient, special-purpose algorithms for particular classes of problems, including combinatorial partitioning problems (clustering, max-cut, etc) \citep{balcan2017learning}, branching strategies in tree search \citep{balcan2018learning}, and various algorithm selection problems \citep{Vitercik2018}. Nevertheless, this vein of work takes a perspective similar to our own and demonstrates that algorithm configuration has moved decisively from being solely the province of heuristic methods to being a topic for rigorous theoretical study.

\section{Model}

We define an algorithm configuration problem by the 4-tuple $(N, \Gamma, R, \kappa_0)$, where these elements are defined as follows. 
$N$ is a family of (potentially randomized)
algorithms, which we call \emph{configurations} to suggest that a single piece of code instantiates each algorithm under a different parameter setting.  We do not assume that different configurations exhibit any sort of performance correlations, and can so capture the case of $n$ distinct algorithms by imagining a ``master algorithm'' with a single, $n$-valued categorical parameter.
Parameters are allowed to take continuous values: $|N|$ can be uncountable. We typically use $i$ to index configurations.
$\Gamma$ is a probability distribution over input instances.
When the instance distribution is given implicitly by a finite benchmark set, let $\Gamma$ be the uniform distribution over this set.  
We typically use $j$ to index (input instance, random seed) pairs, to which we will hereafter refer simply as instances.
$R(i,j)$ is the execution time when configuration $i \in N$ is run on input instance $j$.  Given some value of $\theta > 0$, we define $R(i,j,\theta) = \min\{R(i,j), \theta\}$, the runtime capped at $\theta$.
$\kappa_0 > 0$ is a constant such that $R(i,j) \geq \kappa_0$ for all configurations $i$ and inputs $j$.

For any timeout threshold $\theta$,
let $R_\theta(i) = \E_{j \sim \Gamma}[R(i,j,\theta)]$ denote the average $\theta$-capped running time of configuration $i$, over input distribution $\Gamma$.
Fixing some running time $\bar{\kappa} = 2^{\beta} \kappa_0$ that we will never be willing to exceed, 
the quantity $R_{\bar{\kappa}}(i)$ corresponds
to the expected running time of configuration $i$ and
will be denoted simply by $R(i)$. We will write $OPT = \min_i R(i)$.
Given $\epsilon > 0$, a goal is to find $i^* \in N$ such that $R(i^*) \leq (1+\epsilon) OPT$. 
We also consider a relaxed objective, where the running time of $i^*$ is {\em capped} at some threshold value $\theta$ for some small fraction of (instance, seed) pairs $\delta$. 
\begin{definition} \label{def:eps-delta-opt}
A configuration $i^*$ is \emph{$(\epsilon,\delta)$-optimal} if there
exists some threshold $\theta$ such that $R_{\theta}(i^*) \leq
(1+\epsilon) OPT$, and 
$\Pr_{j \sim \Gamma} \big( R(i^*,j) > \theta \big) \leq \delta$.  Otherwise, we say $i^*$ is \emph{$(\epsilon, \delta)$-suboptimal}.
\end{definition}

\section{\mbox{Structured Procrastination with Confidence}}
\label{procrast}

In this section we present and analyze our 
algorithm configuration procedure, which is
based on the ``Structured Procrastination''
principle introduced in~\cite{ijcai17}. 
We call the procedure SPC (Structured
Procrastination with Confidence) because, 
compared with the original Structured Procrastination algorithm,
the main innovation is that instead of 
approximating the running time of each
configuration by taking $\widetilde{O}(1/\eps^2)$
samples for some $\eps$, it approximates
it using a lower confidence bound that becomes
progressively tighter as the number of samples
increases. We focus on the case where $N$, the
set of all configurations, is finite and can be
iterated over explicitly.  Our main result for this case
is given as Theorem~\ref{thm:main}.  In Section~\ref{gamma} we
extend SPC to handle large or
infinite spaces of configurations where full 
enumeration is impossible or impractical.

\subsection{Description of the algorithm}

\begin{wrapfigure}{R}{0.61\textwidth}
\begin{minipage}{.90\textwidth}
\begin{algorithm}[H]
\alginit
\caption{Structured Procrastination w/ Confidence \label{alg.sp1}}
\setcounter{AlgoLine}{0}
	\SetKwInOut{Require}{require}
	\Require{Set $N$ of $n$ algorithm configurations}
	\Require{Lower bound on runtime, $\kappa_0$}
	\BlankLine
	\tcp{\emph{Initializations}}
	$t := 0$ \\
	\For{$i \in N$}{
		$C_i := $ new Configuration Tester for $i$\;
        $C_i.$\texttt{Initialize()}\;
	}
		
	\BlankLine
	\tcp{\emph{Main loop. Run until interrupted.}}
	\Repeat(\hspace{0.4cm}\emph{\small{// \texttt{GetLCB()} returns LCB as described in the text.}}){anytime search is interrupted}{ 
    	$i := \argmin_{i \in N} \ C_i.$\texttt{GetLCB()}\;
		$C_{i}.\texttt{ExecuteStep()}$\;
	}

	\Return{$\;\; i^* = \argmax_{i \in N} \left\{ C_i.\texttt{GetNumActive()} \right\}$}
    \SetKwProg{myClass}{Class}{}{}
    \SetKwFunction{Initialize}{Initialize}
    \SetKwFunction{ExecuteStep}{ExecuteStep}
    \SetKwFunction{ConfigurationTester}{ConfigurationTester}

	\BlankLine
	\tcp{\emph{Configuration Testing Controller.}}
	\myClass{\ConfigurationTester{}}{
    \SetKwProg{myproc}{Procedure}{}{}
    \SetKwFunction{Initialize}{Initialize}
    \SetKwFunction{ExecuteStep}{ExecuteStep}
    \SetKwFunction{GetNumActive}{GetNumActive}

	\Require{Sequence $j_1,j_2,\ldots$ of instances}
	\Require{Global iteration counter, $t$}

	\BlankLine
	\myproc{\Initialize{}}{
        $r := 0$, $\theta := \kappa_0$, $q = 1$\; 
		$Q := $ empty double-ended queue\;
	}
	\BlankLine
	\myproc{\ExecuteStep{}}{
		$t := t + 1$\;
		\eIf(\hspace{0.5cm}\emph{\small{// Replenish queue}}){$|Q| < q$}{
			$r := r + 1$\;
			$\ell := r$\;
		}{
			Remove $(\ell,\theta')$ from head of $Q$\;
			$\theta := \theta'$\;
		}
		\eIf{$\run(i,j_\ell,\theta)$ terminates in time $\tau \leq \theta$}{
			$R_{i\ell\theta} := \tau$\;
		}{
			$R_{i\ell\theta} := \theta$\;
			Insert $(\ell,2 \theta)$ at tail of $Q$\;
		}
		$q := \lceil 25 \log( t \log r ) \rceil$\;
		}
	\BlankLine
	\myproc{\GetNumActive{}}{
		\Return{$r$}\;
}}
\end{algorithm}
\end{minipage}
\end{wrapfigure}

The algorithm is best described in terms 
of two components: a ``thread pool'' of 
subroutines called {\em configuration 
testers}, each tasked with testing one 
particular configuration, and a {\em scheduler}
that controls the allocation of time to
the different configuration testers. 
Because the algorithm is structured in this
way, it lends itself well to parallelization,
but in this section we will present and analyze
it as a sequential algorithm.

Each configuration tester provides, at all times,
a lower confidence bound (LCB) on the average running 
time of its configuration. 
The rule for computing the LCB will be specified
below; it is designed so that  
(with probability tending to 1 as time goes on) the 
LCB is less than or equal to the true average running 
time. The scheduler runs a main loop whose iterations 
are numbered $t=1,2,\ldots$. In each iteration $t$, it
polls all of the configuration testers for their LCBs,
selects the one with the minimum LCB, and passes
control to that configuration tester. The loop iteration
ends when the tester passes control back to 
the scheduler. 
SPC is an anytime algorithm, so the scheduler's main
loop is infinite; if it is prompted to return
a candidate configuration at any time, the algorithm 
will poll each configuration tester
for its ``score'' (described below) and then output the configuration
whose tester reported the maximum score.

The way each configuration tester $i$ operates is best 
visualized as follows.
There is an infinite stream of 
i.i.d.\ random instances $j_1,j_2,\ldots$ that 
the tester processes. Each of them is either {\em completed},
{\em pending} (meaning we ran the configuration
on that instance at least once, but it timed out before
completing), or {\em inactive}. An instance that is completed or pending 
will be called active. Configuration 
tester $i$
maintains state variables $\theta_i$ and $r_i$
such that the following invariants are satisfied 
at all times: (1) the first $r_i$ instances in the stream are active 
    and the rest are inactive; (2) the number of pending instances is at most $q = q(r_i,t) = 
      50 \log(t \log r_i)$; (3) every pending instance has been attempted with timeout $\theta_i$, and
    no instance has been attempted with timeout greater than $2 \theta_i$.
To maintain these invariants, 
configuration tester $i$ maintains a queue
of pending instances, each with a timeout
parameter representing the timeout threshold
to be used the next time the configuration attempts
to solve the instance.
When the scheduler passes control to 
configuration tester $i$, it either runs the pending 
instance at the head of its queue (if the queue has
$q(r_i,t)$ elements) or it selects an inactive
instance from the head of the i.i.d.\ stream and
runs it with timeout threshold $\theta_i$. 
In both cases, if the run exceeds its timeout, 
it is reinserted into the back of the queue 
with the timeout threshold doubled. 

At any time, if configuration tester $i$ is 
asked to return a score 
(for the purpose of selecting a candidate
optimal configuration) it simply outputs $r_i$, the number of active instances.
The logic justifying this choice of score 
function is that the scheduler devotes more
time to promising configurations than to those
that appear suboptimal; furthermore, better
configurations run faster on average and so 
complete a greater number of runs. This dual tendency
of near-optimal configuration testers to be
allocated a greater amount of running time and
to complete a greater number of runs per unit 
time makes the number of active instances a
strong indicator of the quality of a configuration, as we formalize in the analysis. 

We must finally specify how configuration tester $i$
computes its lower confidence bound on $R(i)$; 
see Figure~\ref{fig:1} for an illustration.
Recall that the configuration tester has a state
variable $\theta_i$ and that for every active instance $j$,
the value $R(i,j,\theta_i)$ is already known
because $i$ has either completed instance $j$,
or it has attempted instance $j$ with 
timeout threshold $\theta_i$.  Given some iteration of the algorithm, define $G$
to be the empirical cumulative distribution function (CDF) of $R(i,j,\theta_i)$
as $j$ ranges over all the active instances.
A natural estimation of $R_{\theta_i}(i)$ would be the expectation
of this empirical distribution, $\int_0^{\infty}(1-G(x))dx$.  Our
lower bound will be the expectation of a modified CDF, found by
scaling $G$ non-uniformly toward $1$.
To formally describe the modification we require some definitions.
Here and throughout this paper, we use the notation
$\log(\cdot)$ to denote the base-2 logarithm and
$\ln(\cdot)$ to denote the natural logarithm.
Let $\epsilon(k,r,t) = \sqrt{\frac{9 \cdot 2^{k} \ln(kt)}{r}}.$  For $0 < p < 1$ let 
\begin{equation} \label{eq:betaprt}
  \beta(p,r,t) = 
    \begin{cases}
      \frac{p}{1+\epsilon(\lfloor \log(1/p) \rfloor,r,t)}
      & \mbox{if } \epsilon(\lfloor \log(1/p) \rfloor, r, t) \le 1/2 \\
      0 & \mbox{otherwise.}
    \end{cases}
\end{equation}
Given a function $G : [0,\infty) \to [0,1]$, we let 
$$
  L(G,r,t) = \int_0^{\infty} \beta(1-G(x),r,t) \, dx.
$$
The configuration tester's lower confidence bound is $L(G_i,r_i,t)$, 
where $t$ is the current iteration, $r_i$ is the number of active instances,
and $G_i$ is the empirical CDF of $R(i,j,\theta_i)$.

To interpret this definition and Equation \eqref{eq:betaprt}, think of $p$ as the value of $1-G(x)$ for some $x$,
and $\beta(p,r,t) \leq p$ as a scaled-down version of $p$.  The scaling factor we use, $(1 + \epsilon(k,r,t))$, depends 
on the value of $p$; specifically, it increases with $k = \lfloor \log(1/p) \rfloor$.  In other words, we scale $G(x)$ more aggressively
as $G(x)$ gets closer to $1$.  If $p$ is too small as a function of $r$ and $t$, 
then we give up on scaling it and instead set it all the way to $\beta(p,r,t) = 0$.  To see this, note that for $k$ such that $2^{-k} \leq p < 2^{1-k}$,
if $k$ is large enough then we will have that $\epsilon(k,r,t) > 1/2$ so the second case of Equation~\eqref{eq:betaprt} applies.

We also note that $L(G_i,r_i,t)$ can be explicitly computed.
Observe that $G_i(x)$ is actually a 
step function with at most $r_i$ steps
and that $G_i(x)=1$ for $x > \theta_i$,
so the integral defining $L(G_i,r_i,t)$ is actually a 
finite sum that can be computed in $O(r_i)$
time, given a sorted list of the elements
of $\{R(i,j,\theta_i) \mid j \mbox{ active}\}$. Example \ref{example:runtime} illustrates the gains SPC can offer over SP.

\begin{example}\label{example:runtime}
 Suppose that there are two configurations: one that takes $100ms$ on every input and another that takes $1000ms$. With $\kappa_0 = 1 ms$, $\epsilon=0.01$, and $\zeta=0.1$, SP will set the initial queue size of each configuration to be at least\footnote{The exact queue size depends on the number of active instances, but this bound suffices for our example.} $7500$, because the queue size is initialized with a value that is at least $12 \epsilon^{-2} \ln(3 \beta n / \zeta)$. It will run each configuration $7500$ times with a timeout of $1ms$, then it will run each of them $7500$ times with a timeout of $2ms$, then $4ms$, and so on, until it reaches $128ms$. At that point it exceeds $100ms$, so the first configuration will solve all instances in its queue. However, for the first $2 \cdot 7500 \cdot (1+2+4+\cdots+64) = 1.9 \times 10^6$ milliseconds of running the algorithm---more than half an hour---essentially nothing happens: SP obtains no evidence of the superiority of the first configuration. 

In contrast, SPC maintians more modest queue sizes, and thus runs each configuration on fewer instances before running them with a timeout of $128ms$, at which point it can distinguish between the two. During the first $5000$ iterations of SPC, the size of each configuration's instance queue is at most $400$. This is because $r_i \le t$, and $t \le 5000$, so $q_i \le 25 \log(5000 \log(5000)) < 400$. Further, observe that $5000$ iterations is sufficient for SPC to attempt to run both configurations on some instance with a cutoff of $128ms$, since each configuration will first run at most $400$ instances with cutoff $1ms$, then at most $400$ instances with cutoff $2ms$, and so on. Continuing up to $64ms$, for both configurations, takes a total of $2 \cdot \log(64) \cdot 400 = 4800 < 5000$ iterations. Thus, it takes at most $2 \cdot 400 \cdot (1 + 2 + 4 + ... + 64) = 101,600$ milliseconds (less than two minutes) before SPC runs each configuration on some instance with cutoff time $128ms$. We see that SPC requires significantly less time---in this example, almost a factor of $20$ less---to reach the point where it can distinguish between the two configurations. 
\end{example}

\begin{figure}[t]
\hspace{1.5cm}
\floatbox[{\capbeside\thisfloatsetup{capbesideposition={left,top},capbesidewidth=5cm}}]{figure}[\FBwidth]
{\caption{\small An illustration of how we compute the lower bound on a configuration's average runtime. The distribution of a given configuration's true runtime is $F(x)$;  the empirical CDF, $G(x)$, constitutes observations sampled from $F(x)$ and censored at $\theta$. The configuration's expected runtime, the quantity we want to estimate, is the (blue) shaded region above curve $F(x)$. Our high-probability lower bound on this quantity is the (green) area above $G(x)$, scaled towards 1 as described in Equation \eqref{eq:betaprt}.}\label{fig:1}}
{\vspace{-.31cm}{\hspace{.5cm}
\begin{picture}(225,160)
\put(0,10){\includegraphics[width=3.25in]{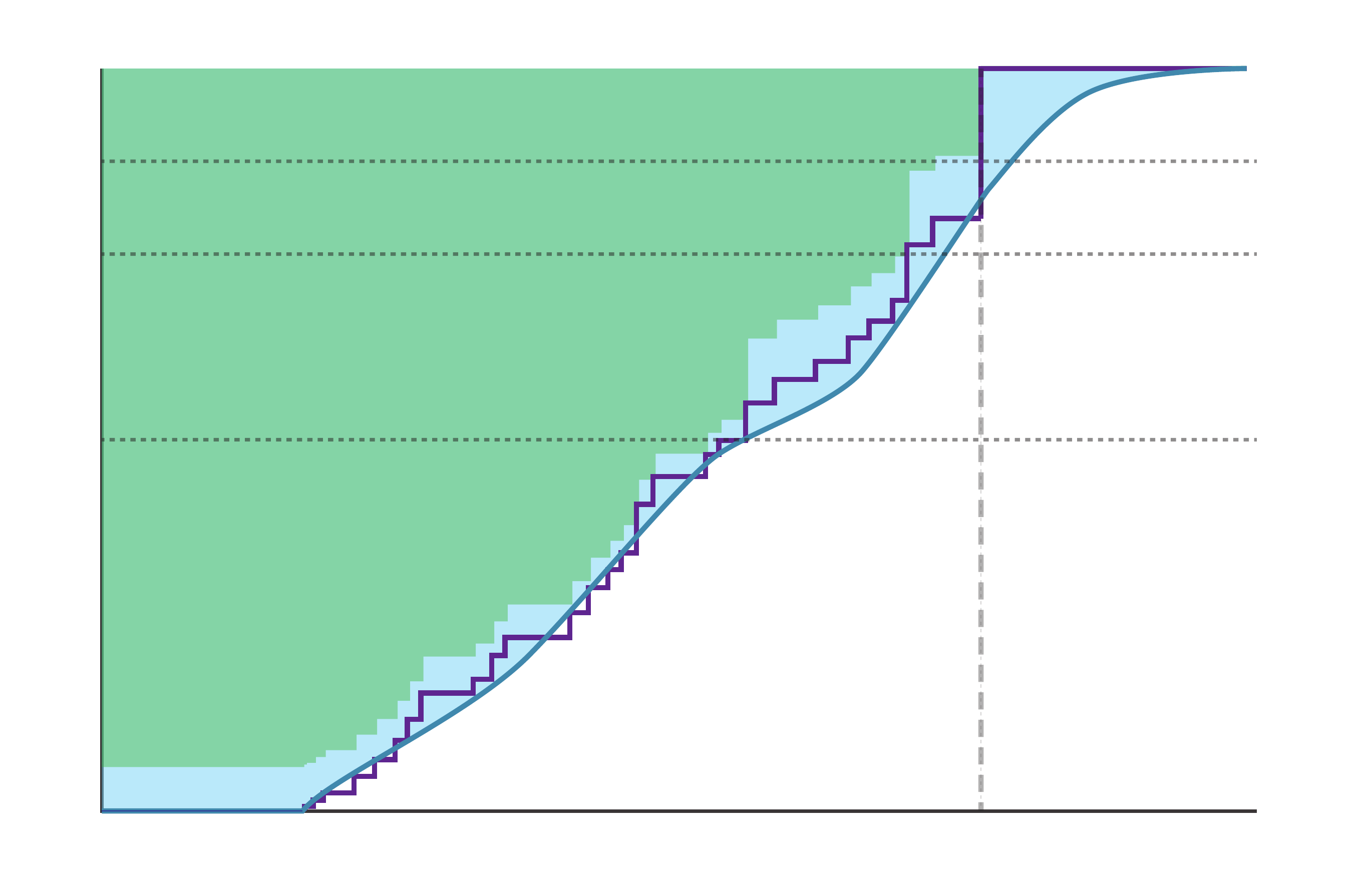}}
\put(100,10){Runtime}
\put(48,14){$\kappa_0$}
\put(165,14){$\theta$}
\put(-2,25){\mbox{\rotatebox{90}{\small Probability of solving an instance}}}
\put(10,20){$0$}
\put(8,84){\nicefrac{1}{2}}
\put(8,115){\nicefrac{3}{4}}
\put(8,131){\nicefrac{7}{8}}
\put(10,147){$1$}
\put(186,139){$F(x)$}
\put(102,48){$G(x)$}
\end{picture}}}
\vspace{-.8cm}
\end{figure}

\subsection{Justification of lower confidence bound}
\label{sec:lcb}

In this section we will show that for any configuration $i$
and any iteration $t$, with probability $1-O(t^{-5/4})$ the 
inequality $L(G_i,r_i,t) \le R(i)$ holds. 
Let $F_i$ denote the cumulative distribution function
of the running time of configuration $i$. Then 
$R(i) = \int_0^{\infty} 1 - F_i(x) \, dx$, so in order
to prove that $L(G_i,r_i,t) \le R(i)$ with high probability
it suffices to prove that, with high probability,
for all $x$ the inequality $\beta(1-G_i(x),r_i,t) \le 1-F_i(x)$ holds.
To do so we will apply a multiplicative error estimate 
from empirical process theory due to~\cite{wellner78}.
This error estimate can be used to derive  the following error bound in our setting.

\begin{lemma} \label{lem:wellner}
  Let $x_1,\ldots,x_n$ be independent random samples
  from a distribution with cumulative distribution function $F$, and $G$ their empirical CDF. For $0 \le b \le 1$, $x \geq 0$,
  and $0 \le \eps \le 1/2$ define the events $\mathcal{E}_1(b,x) = \{ 1-G(x) \ge b \}$ and $\mathcal{E}_2(\epsilon,x) = \big\{ \frac{1-G(x)}{1+\eps} > 1-F(x) \big\}$. Then we have
    $\Pr \left( \exists\ x \text{ s.t. } \mathcal{E}_1(b,x) \mbox{ and } \mathcal{E}_2(\epsilon,x) \right) \le \exp(-\tfrac14 \eps^2 n b).$
\end{lemma}
To justify the use of $L(G_i,r_i,t)$ as a lower confidence
bound on $R(i)$, we apply Lemma~\ref{lem:wellner} with $b=2^{-k}, \, n = r$ and $\eps = \eps(k,r,t)$. With these parameters, 
$\frac14 \eps^2 n b = \frac94 \ln(kt)$, hence the 
lemma implies the following for all 
$k, r, t$:
\begin{equation} \label{eq:lgrt.1}
  \Pr \left( \exists\ x \text{ s.t. } \mathcal{E}_1(2^{-k},x) \mbox{ and } \mathcal{E}_2(\eps(k,r,t),x) \right) \le (kt)^{-9/4}.
\end{equation}
The inequality is used in the following
proposition to show that $L(G_i,r_i,t)$ is a lower bound on $R(i)$ with high probability.
\begin{lemma}\label{lem:lgrt} For each configuration tester, $i$, 
and each loop iteration $t$,
\begin{equation} \label{eq:lgrt.2}
  \Pr \left( \exists x \mbox{ s.t. }
    \beta(1-G_i(x),r_i,t) > 1 - F_i(x) \right) =
    O(t^{-5/4}).
\end{equation}
Consequently $\Pr \left( L(G_i,r_i,t) > R(i) \right) = O(t^{-5/4})$.
\end{lemma}

\subsection{Running time analysis}

Since SPC spends less time running bad configurations, we are able to show an improved runtime bound over SP. Suppose that $i$ is $(\eps,\delta)$-suboptimal. 
We bound the expected amount
of time devoted to running $i$ during the first
$t$ loop iterations. We show that this quantity
is $O(\eps^{-2} \delta^{-1} \log(t \log (1/\delta) ) )$.
Summing over $(\eps,\delta)$-suboptimal 
configurations yields our main result, which is that Algorithm~\ref{alg.sp1} is extremely unlikely to return an $(\epsilon,\delta)$-suboptimal configuration once its runtime exceeds the average runtime of the best configuration by a given factor.
Write $B(t,\eps,\delta) = \eps^{-2} \delta^{-1} \log(t \log(1/\delta))$. 

\begin{theorem}\label{thm:main}
Fix $\eps$ and $\delta$ and let $S$ be the set of $(\eps,\delta)$-optimal configurations.  For each $i \not\in S$ suppose
that $i$ is $(\eps_i,\delta_i)$-suboptimal, with $\eps_i \geq \eps$ and $\delta_i \geq \delta$.
Then if the time spent running SPC is 
\[\Omega\bigg(R(i^*) \bigg( |S| \cdot B(t,\eps,\delta) + \sum_{i \not\in S} B(t,\eps_i,\delta_i) \bigg) \bigg),\]
where $i^*$ denotes an optimal configuration, then SPC will return an $(\eps,\delta)$-optimal configuration when it is terminated, with high probability in $t$.
\end{theorem}

 Rather than having an additive  $O(\epsilon^{-2} \delta^{-1})$ term for each of $n$ configurations considered (as is the case with SP), the bound in Theorem~\ref{thm:main} has a term of the form  $O(\epsilon_i^{-2} \delta_i^{-1})$, for each configuration $i$ that is not $(\epsilon, \delta)$-optimal, where $\epsilon_i^{-2} \delta_i^{-1}$ is as small as possible. This can be a significant improvement in cases where many configurations being considered are far from being $(\epsilon, \delta)$-optimal.
To prove Theorem~\ref{thm:main}, 
we will make use of the following lemma, which bounds the time spent running configuration $i$ in terms of its lower
confidence bound and number of active instances.

\begin{lemma} \label{lem:time-bound}
  At any time, if the configuration tester for
  configuration $i$ has $r_i$ active instances and
  lower confidence bound $L_i$, then the total 
  amount of running time that has been spent 
  running configuration $i$ is at most 
  $9 r_i L_i$.
\end{lemma}

The intuition is that because execution timeouts are successively doubled,
the total time spent running on a given input instance $j$ is not much more than the time of the
most recent execution on $j$.  But if we take an average over all active $j$, the total time spent on
the most recent runs is precisely $r$ times the average runtime under the empirical CDF.  The result then follows
from the following lemma, Lemma~\ref{lem:L.G.gap}, which shows that $L_i$ is at least a constant times this empirical average runtime.

\begin{lemma}\label{lem:L.G.gap}
At any iteration $t$, if the configuration tester for configuration $i$ has $r_i$ active instances and $G_i$ is the empirical CDF for $R(i,j,\theta_i)$, then
$L(G_i,r_i,t) \ge \frac23 \int_0^{\theta_i} (1-G_i(x)) \, dx.$
\end{lemma}

Given Lemma~\ref{lem:time-bound}, it suffices to argue that a sufficiently suboptimal configuration
will have few active instances.  This is captured by the following lemma.

\begin{lemma}\label{lem.lem1}
  If configuration $i$ is $(\eps_i,\delta_i)$-suboptimal
  then at any iteration $t$, the expected number of active
  instances for configuration tester $i$ is bounded by
  $O(\eps_i^{-2} \delta_i^{-1} \log(t \log(1/\delta_i)))$
  and the expected amount of time spent running
  configuration $i$ on those instances is bounded 
  by $O(R(i^*) \cdot \eps_i^{-2} \delta_i^{-1} \log(t \log(1/\delta_i)))$
  where $i^*$ denotes an optimal configuration.
\end{lemma}

Intuitively, Lemma~\ref{lem.lem1} follows because in order for the algorithm to select a
suboptimal configuration $i$, it must be that the lower bound for $i$ is less
than the lower bound for an optimal configuration.  Since the lower bounds
are valid with high probability, this can only happen if the lower bound for 
configuration $i$ is not yet very tight.  Indeed, it must be significantly less 
than $R_\phi(i)$ for some threshold $\phi$ with $\Pr_j(R(i,j) > \phi) \geq \delta_i$.
However, the lower bound cannot remain this loose for long: once the threshold
$\theta$ gets large enough relative to $\phi$, and we take sufficiently many samples
as a function of $\epsilon_i$ and $\delta_i$, standard concentration bounds will imply that the
empirical CDF (and hence our lower bound) will approximate the true runtime 
distribution over the range $[0,\phi]$.  Once this happens, the lower bound will exceed
the average runtime of the optimal distribution, and configuration $i$ will stop receiving
time from the scheduler.

Lemma \ref{lem.lem1} also gives us a way of determining $\epsilon$ and $\delta$ from an empirical run of SPC.  
If SPC returns configuration $i$ at time $t$, then by Lemma~\ref{lem.lem1} $i$ will not be $(\epsilon,\delta)$-suboptimal for any $\epsilon$ and $\delta$ for which $r_{i} = \Omega(\epsilon^{-2}\delta^{-1} \log(t \log(1/\delta)))$, where $r_i$ is the number of active instances for $i$ at termination time.  Thus, given a choice of $\epsilon$ and the value of $r_i$ at termination, one can solve to determine a $\delta$ for which $i$ is guaranteed to be $(\epsilon,\delta)$-optimal.  See Appendix \ref{app:ed} for further details.

Given Lemma~\ref{lem.lem1}, Theorem~\ref{thm:main} follows from a straightforward counting argument; see Appendix B.

\section{Handling Many Configurations}
\label{gamma}

Algorithm~\ref{alg.sp1} assumes a fixed set $N$ of $n$ possible configurations.  In practice, these configurations are often determined by the settings of dozens or even hundreds of parameters, some of which might have continuous domains.  In these cases, it is not practical for the search procedure to take time proportional to the number of all possible configurations. 
However, like Structured Procrastination, the SPC procedure can be modified to handle such cases.  What follows is a brief discussion; due to space constraints, the details are provided in the supplementary material.  

The first idea is to sample a set $\hat{N}$ of $n$ configurations from the large (or infinite) pool, and run Algorithm~\ref{alg.sp1} on the sampled set.  This yields an $(\epsilon,\delta)$-optimality guarantee with respect to the best configuration in $\hat{N}$.  Assuming the samples are representative, this corresponds to the top $(1/n)$'th quantile of runtimes over all configurations.  We can then imagine running instances of SPC in parallel with successively doubled sample sizes, appropriately weighted, so that we make progress on estimating the top $(1/2^k)$'th quantile simultaneously for each $k$.  This ultimately leads to an extension of Theorem~\ref{thm:main} in which, for any $\gamma > 0$, one obtains a configuration that is $(\epsilon,\delta)$-optimal with respect to $\OPT^\gamma$, the top $\gamma$-quantile of configuration runtimes.  This method is anytime, and the time required for a given $\epsilon$, $\delta$, and $\gamma$ is (up to log factors) $\OPT^\gamma \cdot \tfrac{1}{\gamma}$ times the expected minimum time needed to determine whether a randomly chosen configuration is $(\epsilon,\delta)$-suboptimal relative to $\OPT^\gamma$.

\section{Experimental Results}
\label{experiment}

\begin{figure}[t]
\hspace{-.5cm}
\floatbox[{\capbeside\thisfloatsetup{capbesideposition={left,top},capbesidewidth=4.0cm}}]{figure}[\FBwidth]
{\vspace{-.2cm} \caption{\small Mean runtimes for solutions returned by SPC after various amounts of compute time (blue line), and for those returned by LB for different $\epsilon, \delta$ pairs (red points). For LB, each point represents a different $\epsilon, \delta$ combination. Its size represents the value of $\epsilon$, and its color intensity represents the value of $\delta$. SPC is able to find a good solution relatively quickly. Different $\epsilon, \delta$ pairs can lead to drastically different runtimes, while still returning the same configuration. The $x$-axis is in log scale.}\label{fig:experiemnt}}
{\includegraphics[width=3.7in]{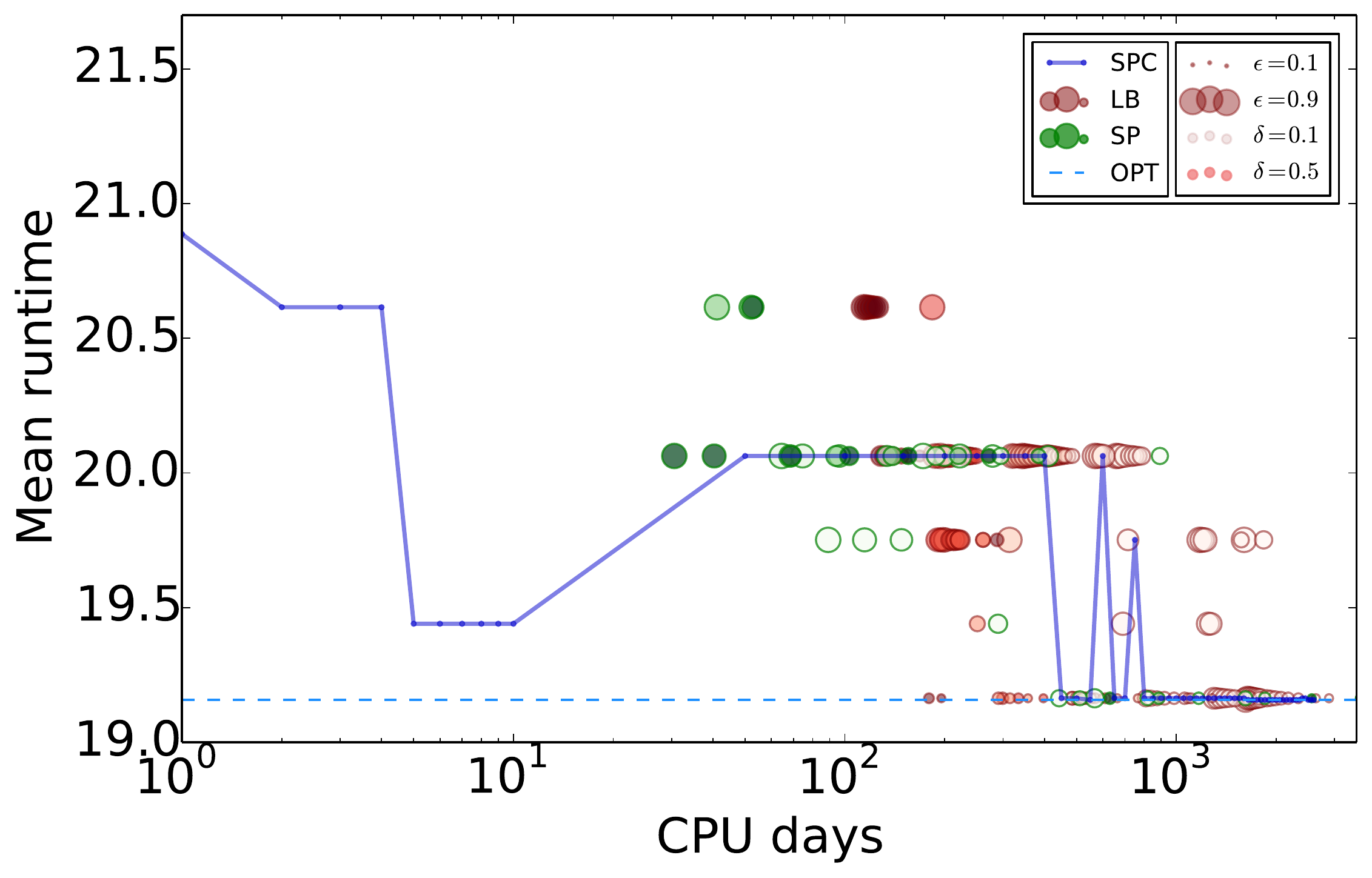}}
\vspace{-.7cm}
\end{figure}

We experiment\footnote{Code to reproduce experiments is available at \url{https://github.com/drgrhm/alg_config}} with SPC on the benchmark set of runtimes generated by \citet{weisz2018leapsandbounds} for testing \textsc{LeapsAndBounds}. This data consists of pre-computed runtimes for 972 configurations of the \texttt{minisat} \citep{sorensson2005minisat} SAT solver on 20118 SAT instances generated using CNFuzzDD\footnote{http://fmv.jku.at/cnfuzzdd/}. A key difference between SPC and LB is the former's anytime guarantee: unlike with LB, users need not choose values of $\epsilon$ or $\delta$ in advance. Our experiments investigate the impact of this property. To avoid conflating the results with effects due to restarts and their interaction with the multiplier of $\theta$, all the times we considered were for the non-resuming simulated environment.  

Figure \ref{fig:experiemnt} compares the solutions returned by SPC after various amounts of CPU compute time with those of LB and SP for different $\epsilon, \delta$ pairs chosen from a grid with $\epsilon \in [0.1, 0.9]$ and $\delta \in [0.1, 0.5]$. The $x$-axis measures CPU time in days, and the $y$-axis shows the expected runtime of the solution returned (capping at the dataset's max cap of $900s$). The blue line shows the result of SPC over time. The red points show the result of LB for different $\epsilon, \delta$ pairs, and the green points show this result for SP. The size of each point is proportional to $\epsilon$, while the color is proportional to $\delta$. 

We draw two main conclusions from Figure \ref{fig:experiemnt}. First, SPC was able to find a reasonable solution after a much smaller amount of compute time than LB. After only about 10 CPU days, SPC identified a configuration that was in the top 1\% of all configurations in terms of max-capped runtime, while runs of LB took at least 100 CPU days for every $\epsilon,\delta$ combination we considered. Second, choosing a good $\epsilon$, $\delta$ combination for LB was not easy. One might expect that big, dark points would appear at shorter runtimes, while smaller, lighter ones would appear at higher runtimes. However, this was not the case. Instead, we see that different $\epsilon, \delta$ pairs led to drastically different total runtimes, often while still returning the same configuration. Conversely, SPC lets the user completely avoid this problem. It settles on a fairly good configuration after about 100 CPU days. If the user has a few hundred more CPU days to spare, they can continue to run SPC and eventually obtain the best solution reached by LB, and then to the dataset's true optimal value after about 525 CPU days. However, even at this time scale many $\epsilon, \delta$ pairs led to worse configurations being returned by LB than SPC. 

\section{Conclusion}
\label{conclusion}
We have presented Structured Procrastination with Confidence, an approximately optimal procedure for algorithm configuration. SPC is an anytime algorithm that uses a novel lower confidence bound to select configurations to explore, rather than a sample mean. As a result, SPC \emph{adapts} to problem instances in which it is easier to discard poorly-performing configurations. We are thus able to show an improved runtime bound for SPC over SP, while maintaining the anytime property of SP. 

We compare SPC to other configuration procedures on a simple benchmark set of SAT solver runtimes, and show that SPC's anytime property can be helpful in finding good configurations, especially early on in the search process. However, a more comprehensive empirical investigation is needed, in particular in the setting of many configurations. Such large-scale experiments will be a significant engineering challenge, and we leave this avenue to future work. 

\bibliographystyle{icml2019}
\bibliography{main}

\newpage
\appendix
\section{Runtime Variation in Practice}
\label{sec:empirical} 

\begin{figure*}[t]\centering
\vbox{
\includegraphics[width=.8\textwidth]{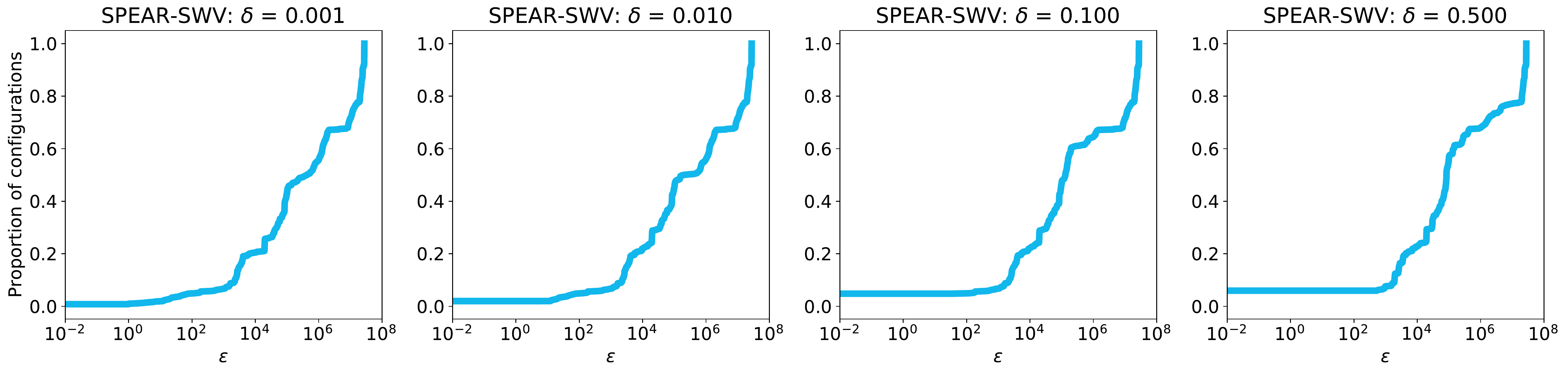}
\includegraphics[width=.6\textwidth]{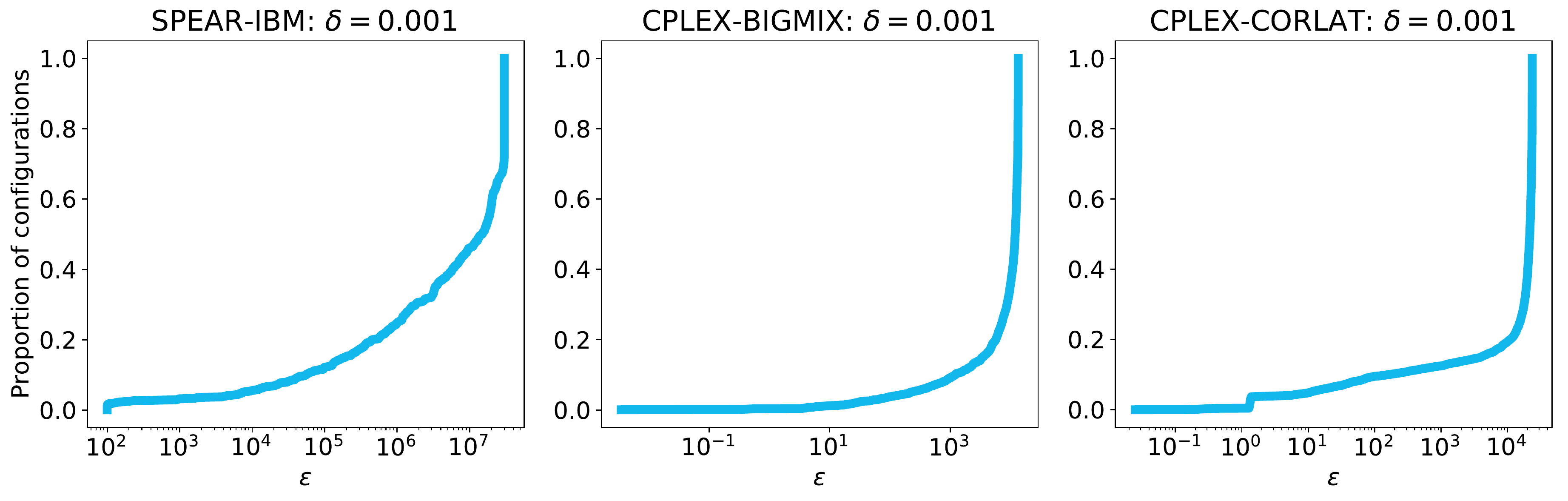}
}
\caption{Empirical runtime variation for different solvers and input distributions. For given $\delta$, each plot shows the fraction of configurations which are $(\epsilon, \delta)$-optimal for different values of $\eps$; data from \citet{hutter2014algorithm}. \textbf{(top)} SPEAR SAT solver configurations on SWV for various $\delta$. \textbf{(bottom)} SPEAR on IBM instances and CPLEX MIP solver on various distributions, for fixed values of $\delta$.}
\label{fig:epsilon}
\end{figure*}

Unlike Structured Procrastination, SPC is designed to perform better when  relatively few configurations are much faster on average than all others. It is thus worth asking whether this occurs in practice. We examined publicly available data from \citet{hutter2014algorithm} (see~{\smaller{\url{http://www.cs.ubc.ca/labs/beta/Projects/EPMs}}}), which studied the performance of two very different heuristic solvers (CPLEX, for mixed integer programs; and SPEAR, for satisfiability) on a total of 6 different benchmark distributions of practical problem instances; we investigate two distributions for each solver here.
These observations were generated by randomly sampling from solvers' parameter spaces, just as SPC does; runs were given a captime of 300 seconds. We modified the data so that capped runs were recorded as having finished in 300 seconds (to bias against reporting variation in average runtimes across configurations). 

We found a great deal of variation in average runtime across configurations; see Figure \ref{fig:epsilon}. Each plot corresponds to a specific value of $\delta$, and shows the CDF of the smallest value of $\epsilon$ for which each configuration remains $(\epsilon, \delta)$-optimal. The first row of this figure is based on different configurations of the SPEAR solver on SWV instances, with different figures corresponding to different $\delta$ values. Each figure's $x$-axis corresponds to $\eps$ values (on a log scale); the $y$-axis reports the fraction of configurations that were $(\eps, \delta)$-optimal for the given values of $\eps$ and $\delta$. Observe that many configurations (between 1\% and 6\%) tie for being best for a range of small $\eps$ values: this is because $\kappa_0 =0.01$ in this setting, so fast configurations were often indistinguishable. This fraction grows with $\delta$: more configurations become indistinguishable when we sanitize their performance on larger fractions of instances. In the bottom row, the point in each graph where the CDF spikes upward corresponds to configurations where most instances were capped; thus, these graphs understate the true runtime variation.

What do these results mean for SPC? Consider SPEAR--SWV with $\delta=.5$. Only about 5\% of configurations are optimal for $\epsilon$ less than about 100: i.e.,  even when capped runs are reported as having finished, 95\% of configurations take at least 100 times longer than an optimal configuration. SPC will easily discard these configurations, allocating very little time to refining their estimates. Broadly, we see a similar pattern across the other solver--distribution pairs.

\section{Omitted Proofs}

\subsection{Proof of Lemma~\ref{lem:wellner}}

Recall the statement of the lemma.  Let $x_1,\ldots,x_n$ be independent random samples
  from a distribution with cumulative distribution function $F$, and $G$ their empirical CDF. For $0 \le b \le 1$, $x \geq 0$,
  and $0 \le \eps \le 1/2$ define the events $\mathcal{E}_1(b,x) = \{ 1-G(x) \ge b \}$ and $\mathcal{E}_2(\epsilon,x) = \big\{ \frac{1-G(x)}{1+\eps} > 1-F(x) \big\}$. Then we have
  \[
    \Pr \left( \exists\ x \text{ s.t. } \mathcal{E}_1(b,x) \mbox{ and } \mathcal{E}_2(\epsilon,x) \right) \le \exp(-\tfrac14 \eps^2 n b).
  \]

We show how this result follows directly from a bound of Wellner~\citep{wellner78}. 
The {\em uniform empirical process} is the random
functon $\Gamma_n : [0,1] \to [0,1]$ defined by
drawing $n$ independent random samples 
$\xi_1,\ldots,\xi_n$ from the uniform 
distribution on $[0,1]$ and 
letting $\Gamma_n$ denote their
empirical CDF, i.e.\ the
cumulative distribution function of the
uniform distribution on $\{\xi_1,\ldots,\xi_n\}$.
Its left-continuous inverse $\Gamma_n^{-1}$
is defined by $\Gamma_n^{-1}(t) = \inf \{s \mid \Gamma_n(s) \ge t \}.$
Lemma~2(i) of~\citep{wellner78} asserts that 
for all $\lambda \ge 1$ and $0 \le b \le 1$,
\[
  \Pr \left( \lambda \le \sup_{b \le t \le 1} 
  \left\{ \frac{t}{\Gamma_n^{-1}(t)} \right\} \right)
  \le \exp(- n b f(1/\lambda))
\]
where $f(x) = x + \ln(1/x) - 1$. Reinterpreting this
using the substitutions $t = \Gamma_n(s)$ and 
$\lambda = 1+\eps$, and making use of the 
inequality $f \left( \frac{1}{1+\eps} \right) \ge \frac14 \eps^2$
for $0 \le \eps \le 1/2$, we get
\begin{align*}
\Pr \left( 1+\eps \le \sup \left\{ \left. 
             \frac{\Gamma_n(s)}{s} \right| \Gamma_n(s) \ge b \right\} \right)
  \le \exp( - \tfrac14 \eps^2 n b ),
  \\
  \forall \, 0 \le \eps \le 1/2, 0 \le b \le 1 \quad
\end{align*}
If $x_1,x_2,\ldots,x_n$ are i.i.d.\ samples drawn from an
atomless distribution with cumulative distribution function $F$,
then the numbers $F(x_1),\ldots,F(x_n)$ are independent
uniformly distributed random samples $[0,1]$,
as are $1-F(x_1),\ldots,1-F(x_n)$. 
Hence if $G$ denotes the empirical CDF of the 
samples $x_1,\ldots,x_n$, then both of the random functions 
$1 - G(F^{-1}(1-s))$ and $G(F^{-1}(s))$ are uniform 
empirical processes. Applying Wellner's Lemma 2(i),
and substituting $s = 1-F(x)$, we obtain Lemma~\ref{lem:wellner}.

\subsection{Proof of Lemma~\ref{lem:lgrt}}

Recall the statement of the lemma: For each configuration tester, $i$, 
and each loop iteration $t$,
\begin{equation*}
  \Pr \left( \exists x \mbox{ s.t. }
    \beta(1-G_i(x),r_i,t) > 1 - F_i(x) \right) =
    O(t^{-5/4}).
\end{equation*}
Consequently $\Pr \left( L(G_i,r_i,t) > R(i) \right) = O(t^{-5/4})$.

\begin{proof}
Sum inequality~\eqref{eq:lgrt.1} 
over $k = 1,2,\ldots$ and 
$r_i = 1,2,\ldots,t$, and use the fact 
that $\sum_{k \ge 1} k^{-9/4} < \infty$,
to deduce inequality~\eqref{eq:lgrt.2}.
Integrate over $0 < x < \infty$ 
to derive the final inequality.
\end{proof}

\subsection{Proof of Lemma~\ref{lem:L.G.gap}}

Recall the statement of the lemma: at any iteration $t$, if the configuration tester for configuration $i$ has $r$ active instances and $G$ is the empirical CDF for $R(i,j,\theta)$, then
\[ L(G,r,t) \ge \frac23 \int_0^{\theta} (1-G(x)) \, dx. \]

\begin{proof}
  Recalling that $L(G,r,t) = \int_0^{\infty} \beta(1-G(x),r,t) \, dx$, 
  it suffices to show that 
  \begin{equation}
  \label{eq.lem}
      \beta(1-G(x),r,t) \ge \frac23 (1-G(x))
  \text{ for all $x \le \theta$}. 
  \end{equation}
  
  To see why \eqref{eq.lem} holds, note that $1-G(\theta) = q(r,t)/r$ because 
  $q(r,t)/r$ is the fraction of pending instances 
  and they all have $R(i,j) \ge \theta$.
  Since $1-G(x)$ is a non-increasing function of $x$,
  this implies that $1-G(x) \ge q(r,t)/r$ for all 
  $0 \le x \le \theta$. 
  
     Recalling the formula for $\beta(p,r,t)$, it is clear
     that 
     \eqref{eq.lem}
     is equivalent to claiming that 
     $\eps(k,r,t) \le 1/2$ whenever 
     $x \le \theta$ and $2^{-k} < 1-G(x) \le 2^{1-k}$.
     Since $\eps(k,r,t)$ is an increasing function of $k$,
     and $1-G(x) \ge q(r,t)/r$, 
     it suffices to prove that $\eps(k,r,t) \le 1/2$ when
     $k = \lceil \log(r/q(r,t)) \rceil$. For this value
     of $k$ we have $\eps(k,r,t) \le \frac12$ as desired.
\end{proof}

\subsection{Proof of Lemma~\ref{lem:time-bound}}

Recall the statement of the lemma: at any time, if the configuration tester for
  configuration $i$ has $r$ active instances and
  lower confidence bound $L$, then the total 
  amount of running time that has been spent 
  running configuration $i$ is at most 
  $9 r L$.

\begin{proof}
    For each active instance $j$, the 
  total time spent running $i$ on $j$ is
  less than $6 \cdot R(i,j,\theta)$. This
  is because the doubling of timeout thresholds
  ensures that the time spent on all previous runs of
  $(i,j)$, combined, is at most twice the amount
  of time spent on the most recent run, which is
  at most $R(i,j,2\theta)$. Hence, the time spent on $j$ is at most $3 \cdot R(i,j,2\theta) \le 6 \cdot R(i,j,\theta)$
  Combining these bounds as $j$ ranges over active 
  instances, the total time spent running $i$ in the
  first $t$ iterations satisfies
  \begin{equation} \label{eq:6r} 
  \left( \mbox{total time spent running } i \right) 
  \le 6r \int_0^{\theta} (1 - G(x)) \, dx, 
  \end{equation}
  since the integral represents the empirical
  average of $R(i,j,\theta)$ over the
  active instances $j$. 
  The proof now follows from Lemma~\ref{lem:L.G.gap}.
\end{proof}

\subsection{Proof of Lemma~\ref{lem.lem1}}

Recall the statement of the lemma:
  if configuration $i$ is $(\eps_i,\delta_i)$-suboptimal
  then at any iteration $t$, the expected number of active
  instances for configuration tester $i$ is bounded by
  $O(\eps_i^{-2} \delta_i^{-1} \log(t \log(1/\delta_i)))$
  and the expected amount of time spent running
  configuration $i$ on those instances is bounded 
  by $O(R(i^*) \cdot \eps_i^{-2} \delta_i^{-1} \log(t \log(1/\delta_i)))$
  where $i^*$ denotes an optimal configuration.

\begin{proof}
We claim that if $i$ is $(\eps,\delta)$-suboptimal, 
then there is a timeout threshold 
$\phi$ and another configuration 
$i^*$ such that $R_{\phi}(i) > (1+\eps) R(i^*)$
and $\Pr_j(R(i,j) > \phi) \geq \delta$.  We prove this formally as Claim~\ref{claim:subopt} below.
Fix such an $i^*$ and $\phi$, and note that we must then
have $R_{\phi}(i) \geq \delta \phi$. In an iteration $t$ when configuration tester $i$
is chosen, let $r,\theta$ denote the internal 
state parameters of configuration tester $i$ and
let $G$ denote its empirical CDF. Similarly, for 
configuration tester $i^*$ let
$r^*,\theta^*$ denote the internal state parameters
and $G^*$ denote the empirical CDF. 
There are two cases to consider. \textbf{(I)} $L(G^*,r^*,t) > R(i^*)$.
  Section~\ref{sec:lcb} showed this event has probability
  $O(t^{-5/4})$. Summing over $t$, in expectation this case
  accounts for only $O(1)$ runs of configuration $i$:
\textbf{(II)} $L(G^*,r^*,t) \le R(i^*)$. 
  In this case, since we know that $R(i^*) <
  (1+\eps)^{-1} R_{\phi}(i)$, and the scheduler's
  selection rule implies that $L(G,r,t) \le L(G^*,r^*,t)$,
  we may conclude that $L(G,r,t) \le (1 + \eps)^{-1} R_{\phi}(i).$
  Letting $k_0 = \lceil \log(1/\delta) \rceil$ and
  recalling the formula for $\eps(k_0,r,t)$, we see
  that for $r > 72 \eps^{-2} \delta^{-1} \log(t \log(1/\delta))$,
  we have $\eps(k_0,r,t) < \eps/2$ and thus
  $\eps(k,r,t) < \eps/2$ for all $k \le k_0$. 
  This means that 
  $$\int_0^{\phi} \beta(1-G(x),r,t) \, dx
    > \frac{2}{2 + \eps} \int_0^\phi (1-G(x)) \, dx.$$
  If we observe that $\E[1-G(x)] = 1-F(x)$ and that
  $\int_0^\phi (1-F(x)) \, dx = R_{\phi}(i)$,
  we see that $L(G,r,t)$ is an average of $r$
  i.i.d.~random samples -- corresponding to scaled draws from
  the empirical distribution $G$ -- each of which lies in the range
  $[0,\phi]$ and has expected
  value greater than $(1+\eps/2)^{-1} R_{\phi}(i)$ (but at most $R_{\phi}(i)$).
  We wish to apply a Chernoff-Hoeffding bound to argue that these samples
  are sufficiently concentrated around their mean.  To this end,
  consider scaling these random variables by $\phi$,
  so that they lie in $[0,1]$ and have expected value at most
  $R_{\phi}(i) / \phi \leq \delta$.  Then
  for $\lambda \ge 1$ and $r > \lambda \cdot 72 
  \eps^{-2} \delta^{-1} \log(t \log(1/\delta))$
  the probability that the empirical average is 
  less than or equal to $(1+\eps)^{-1} R_{\phi}(i)$ is bounded
  above by $e^{-c \lambda}$ by the
  Chernoff-Hoeffding Bound, where $c>0$ is a constant.  (Indeed, as $(1+\eps)^{-1} R_{\phi}(i) \leq (1-\eps/4)(1+\eps/2)^{-1} R_{\phi}(i)$ for all $\epsilon \leq 1$, we can take $c$ to be any constant less than $72 / (2 * 4^2)$, so in particular $c = 2$ suffices.)  Hence, the expected
  number of values of $r$ for which $L(G,r,t) \le 
  (1+\eps)^{-1} R_{\phi}(i)$ is 
  $O(\eps^{-2} \delta^{-1} \log(t \log(1/\delta)))$.

  Let $s_i = 72 \eps_i^{-2} \delta_i^{-1} \log(t \log(1/\delta_i))$.
  The analysis of Case 2 above shows that 
  for $r \ge s_i$
  the probability that 
  we run configuration tester $i$ at least once during
  the first $t$ iterations with a number of active instances 
  equal to $r$ is at most 
  $\exp(-c r / s_i)$.
  Of course, for $r < s_i$ the probability is at most 1.
  Summing over $r=1,2,\ldots$ we obtain the upper bound
  on the expected number of active instances at iteration $t$.
  The bound on combined running time is then derived
  using Lemma~\ref{lem:time-bound}.
\end{proof}

\begin{claim}\label{claim:subopt}
If $i$ is $(\eps,\delta)$-suboptimal, 
then there is a timeout threshold 
$\phi$ and another configuration 
$i^*$ such that $R_{\phi}(i) > (1+\eps) R(i^*)$
and $\Pr_j(R(i,j) > \phi) \geq \delta$.
\end{claim}
\begin{proof}
Choose $i^*$ to be the optimal configuration with respect to uncapped runtime.
By definition, a configuration $i$ is $(\eps,\delta)$-suboptimal if for all $\theta$ such that $\Pr_j(R(i,j) > \theta) \leq \delta$, $R_{\theta}(i) > (1+\eps)R(i^*)$.

Choose $\theta^* = \inf\{\theta \colon \Pr_j(R(i,j) > \theta) \leq \delta\}$.  Then by continuity of $R_\theta(i)$ with respect to $\theta$, we have that $R_{\theta^*}(i) > (1+\eps)R(i^*)$ and $\Pr_j(R(i,j) > \theta) \leq \delta$, as required.
\end{proof}

\subsection{Proof of Theorem~\ref{thm:main}}

Recall the statement of the theorem: fix some $\eps$ and $\delta$, and let $S$ be the set of $(\eps,\delta)$-optimal configurations.  For each $i \not\in S$ suppose
that $i$ is $(\eps_i,\delta_i)$-suboptimal, with $\eps_i \geq \eps$ and $\delta_i \geq \delta$.
Then if the total time spent running SPC is 
\[\Omega\bigg(R(i^*) \bigg( |S| \cdot B(t,\eps,\delta) + \sum_{i \not\in S} B(t,\eps_i,\delta_i) \bigg) \bigg),\]
where $i^*$ denotes an optimal configuration, then SPC will return an $(\eps,\delta)$-optimal configuration when it is terminated, with high probability in $t$.

\begin{proof}
Recall that $B(t,\eps,\delta) = \eps^{-2} \delta^{-1} \log(t \log(1/\delta))$.
Note that $B(t,\eps_i,\delta_i) \leq B(t,\eps,\delta)$ for each $i \not\in S$, by the choice of $\eps_i$ and $\delta_i$.
By Lemma~\ref{lem.lem1}, each $i \not\in S$ runs for a total time of $O(R(i^*) \cdot B(t,\eps_i,\delta_i))$.  Thus, the configurations in $S$ together ran for a total time of at least $\Omega(R(i^*) \cdot |S| \cdot B(t,\eps,\delta))$. At least one configuration $i \in S$ must therefore have run for a total time of $\Omega(R(i^*) \cdot B(t,\eps,\delta))$, and hence the number of active instances for this configuration $i$ is at least $\Omega(B(t,\eps,\delta))$.  As this is larger than the number of active instances for each $i \not\in S$, again by Lemma~\ref{lem.lem1}, we conclude that the configuration with largest number of active instances at termination time lies in $S$, as required.
\end{proof}

\section{Details of Handling Many Configurations}

Like Structured Procrastination, the SPC procedure can be modified to handle cases where the pool of candidates is very large.  Suppose we are given a (possibly infinite) pool $N$ of possible configurations, paired with an implicit probability distribution to allow sampling.  One idea is to sample a set $\hat{N}$ of $n$ configurations, and then run Algorithm~\ref{alg.sp1} on the sampled set.  This would yield an $(\epsilon,\delta)$-optimality guarantee with respect to the best configuration in $\hat{N}$.  
Motivated by this idea, for any $\gamma > 0$, we will define $OPT^\gamma = \inf\{ R \colon \Pr_{i \sim N}[R(i) > R] \leq \gamma \}$.  That is, $OPT^\gamma$ is the top $\gamma$'th quantile of runtimes over all configurations.  
For a fixed $\gamma > 0$, we can sample a set $\hat{N}$ of $O(1/\gamma \cdot \log (1/\gamma))$ configurations, then run Algorithm~\ref{alg.sp1} on the resulting sample.  With high probability (in $1/\gamma$), the optimal configuration from $\hat{N}$, $i^*$, will have $R(i^*) < OPT^\gamma$.  We then achieve a result similar to Theorem~\ref{thm:main}, but with $OPT^\gamma$ in place of $R(i^*)$, and with $\epsilon_i$ and $\delta_i$ now being random variables for each $i \in \hat{N}$.

This discussion assumed that we have advance knowledge of $\gamma$, but we can extend this approach to an anytime guarantee that simultaneously makes progress on every value of $\gamma$.
Suppose that, instead of simply sampling a fixed number of configurations in advance, we ran many instances of SPC in parallel, one for each value of $\gamma = 2^{-1}, 2^{-2}, 2^{-3}, \dotsc$.  For each $k \geq 1$, we draw a sample $\hat{N}_k$ of $\Theta(k \cdot 2^k)$ configurations and execute SPC on set $\hat{N}_k$.  If we share processor time in such a way that process $k$ receives a time share proportional to $1/k^2 = 1/\log(1/\gamma)^2$, then the end result is that the time required to find a configuration that is $(\epsilon,\delta)$-suboptimal with respect to $OPT^\gamma$ increases by a factor of $\log(1/\gamma)^2$, relative to the case in which $\gamma$ was given in advance. Combining these ideas, we arrive at the following extension of Theorem~\ref{thm:main} for the case of large $N$.  Recall that $B(t,\epsilon,\delta)$ is the runtime bound from Lemma~\ref{lem.lem1}.  Given some $i \in N$ and some $\epsilon,\delta,\gamma > 0$, if $i$ is not $(\epsilon,\delta)$-optimal with respect to $\OPT^\gamma$, write 
\begin{align*} V(i, \epsilon,\delta,\gamma,t) = \inf_{\epsilon',\delta' \colon \text{$i$ is $(\epsilon',\delta')$-suboptimal}}\{B(t,\epsilon',\delta')\}.
\end{align*}
Otherwise, set $V(i,\epsilon,\delta,\gamma,t) = B(t,\epsilon,\delta)$.  That is, $V(i,\epsilon,\delta,\gamma,t)$ is the tightest active-instance bound implied by Lemma~\ref{lem.lem1} for configuration $i$.  Write $V(\epsilon,\delta,\gamma,t) = E_{i \sim N}[V(i,\epsilon,\delta,\gamma,t)]$ for the expected number of active instances needed for a randomly sampled configuration.

\begin{theorem}
\label{thm.main.gamma}
Choose any $\epsilon$, $\delta$, and $\gamma$.  Suppose the total time $t$ spent running parallel instances of SPC, as described above, is at least 
$\Omega\left( OPT^{\gamma} \cdot \frac{\log^3(1/\gamma)}{\gamma} \cdot V(\epsilon,\delta,\gamma,t) \right).$
Then, with high probability in $t$, one of the parallel runs of SPC 
(corresponding to $k = \lceil \log(1/\gamma) \rceil$) 
will return an $(\epsilon,\delta)$-optimal configuration with respect to $\OPT^\gamma$.
\end{theorem}

We make two observations.  First, Theorem~\ref{thm.main.gamma} must account for events where the empirical average of $V(i, \epsilon, \delta, \gamma, t)$ over sampled configurations differs significantly from its expectation, $V(\epsilon, \delta, \gamma, t)$.  To bound this difference we use Wellner's theorem, as in Lemma~\ref{lem:wellner}, to show that the empirical CDF is within a constant factor of the true CDF nearly everywhere, except possibly at its lowest values (e.g., those that occur with probability at most $\gamma^{1/2}$).  Even if the empirical distribution varies by a significant amount on these lowest values (up to a factor of $\gamma^{-1/2}$) this will not significantly perturb the empirical average.
Second, note that the bound in Theorem~\ref{thm.main.gamma} is not necessarily monotone in $\gamma$, since $OPT^{\gamma}$ can decrease as $\gamma$ decreases.  This is natural: a broader search is costly, but finding a new fastest configuration will speed up the search procedure.  Thus, even if the user has a certain target value for $\gamma$ in mind, it can be strictly beneficial to allow SPC to search over smaller values of $\gamma$ as well.

\section{Details of Experiments}

Figure \ref{fig:experiemnt} shows the mean runtime of the best configurations found by SPC after various amounts of CPU compute time, and the best configurations returned by LB for different $\epsilon, \delta$ pairs. For SPC we plot points for 1, 2, 3, 5 and 10 CPU days, as well as for every 25 CPU days from 50 to 2600. For the runs of LB, we ran all $\epsilon, \delta$ pairs, with $\epsilon$ chosen from $\{0.1, 0.15, 0.2, 0.25, \dots, 0.9\}$, and $\delta$ chosen from $\{0.1, 0.15, 0.2, 0.25, \dots, 0.5 \}$, for a total of 153 observations. For SP we chose $\epsilon$ from $\{0.1, 0.2, \dots, 0.9\}$, and $\delta$ from $\{0.1, 0.2, \dots, 0.5 \}$,

As in \citet{weisz2018leapsandbounds}, we set the $\zeta$ parameter of LB to 0.1; we used a $\theta$ multiplier of 1.25 and 2 for LB and SPC respectively. As mentioned, all the runtimes we considered were for the simulated environment, which does not allow for restarts. This is the simplest possible scenario in which we can make this comparison. However, an investigation of the effects of restarts, in particular with different values of the $\theta$ multiplier, on these algorithms is an interesting line of future work.

\section{Deriving $\epsilon$ and $\delta$ from an Empirical Execution}\label{app:ed}

A run of SPC returns a configuration $i^*$.  Theorem~\ref{thm:main} provides an $(\epsilon,\delta)$-optimality guarantee, but we note that SPC does not explicitly report the values of $\epsilon$ and $\delta$ to the user.  Indeed, an important feature of SPC is that the quality implications of Theorem~\ref{thm:main} depend on the distribution of running times for the pool of configurations, so for ``easy'' problem instances the actual optimality guarantee attained might be significantly better than in the worst-case.

The following lemma shows that one can infer an improved runtime guarantee from the state of SPC at termination time.  We make use of this approach when evaluating the performance of SPC in experiments.  Roughly speaking, the configuration returned by SPC will be $(\epsilon,\delta)$-optimal when $\epsilon^2\delta$ is inversely proportional to $r_i$, up to logarithmic factors, where recall that $r_i$ is the number of active instances for $i$.

\begin{lemma}
\label{lem:delta}
Suppose that SPC returns configuration $i^*$.  Then for any $\epsilon > 0$, $\delta > 0$, and $\lambda \geq 1$ such that $\epsilon^2 \delta \geq 72 \lambda \log(t \log(1/\delta)) / r_i$, configuration $i^*$ is $(\epsilon,\delta)$-optimal with probability at least $1 - e^{-2\lambda}$.
\end{lemma}
\begin{proof}
Suppose that SPC is terminated at time $t$.  Recall from Lemma~\ref{lem.lem1} that if a configuration $i$ is $(\epsilon,\delta)$-suboptimal, then its expected number of active instances is $O(\epsilon^{-2}\delta^{-1} \log(t \log(1/\delta))$.  Indeed, the proof of Lemma~\ref{lem.lem1} shows something stronger: the probability that the configuration has more than $s_i = 72 \epsilon^{-2}\delta^{-1} \log(t \log(1/\delta))$ active instances at time $t$ is at most $e^{-c}$ for some constant $c$, where in particular taking $c = 2$ suffices.

We conclude from this that if $r_{i^*} \geq 72 \lambda \epsilon^{-2}\delta^{-1} \log(t \log(1/\delta))$, then with probability at least $1-e^{-2\lambda}$ configuration $i^*$ is $(\epsilon,\delta)$-optimal.  In other words, for any $\epsilon$ and $\delta$ such that $\epsilon^2 \delta \geq 72 \lambda \log(t \log(1/\delta)) / r_i$, configuration $i^*$ is $(\epsilon,\delta)$-optimal with probability at least $1-e^{-2\lambda}$.
\end{proof}

By Lemma~\ref{lem:delta}, for any fixed $\epsilon$ we can calculate the $\delta$ for which we have an $(\epsilon,\delta)$-optimality guarantee with, e.g., probability $1/e^{2}$ by setting $\lambda = 1$.  We also note that, up to a constant and a factor of $\log\log(1/\delta)$, this calculation corresponds to the fraction $q_i / r_i$ of pending input instances in the execution of configuration $i^*$ at termination time.

\end{document}